\documentclass[runningheads]{llncs}
\usepackage[ruled,vlined]{algorithm2e}
\usepackage[utf8]{inputenc}
\usepackage{graphicx}
\usepackage{amsmath}
\usepackage{amssymb}
\usepackage{wasysym}
\usepackage{booktabs}
\usepackage{subcaption} 
\usepackage{graphicx}

\begin{document}
\title{Interpretable Models via Pairwise permutations algorithm}

\author{Troy Maasland\inst{2\star} \and João Pereira\inst{1,2}\thanks{Equal contribution to this work} \and Diogo Bastos \inst{1,2}
 \and Marcus de Goffau\inst{1}\and
Max Nieuwdorp\inst{1} \and Aeilko H. Zwinderman\inst{3}\and Evgeni Levin\inst{1,2}}
\authorrunning{T.D. Maasland et al.}
\institute{
Department of Vascular Medicine, Academic Medical Center, University of Amsterdam, 1105 AZ Amsterdam, the Netherlands\and HORAIZON BV, Delft, The Netherlands \and
Department of Clinical Epidemiology, Biostatistics and Bioinformatics,  \\University Medical Center Amsterdam, The Netherlands}
%\emails
%\{d.n.mendesbastos,
%j.p.belopereira, e.levin\}@amsterdamumc.nl}
% \author{
% Diogo Bastos\inst{1,2}
% \and
% Jo\~{a}o Pereira$^{1,2}$
% \and
% Albert K. Groen$^1$\and
%  Erik S. G. Stroes$^1$\And
% Evgeni Levin$^{1,2}$
% \affiliations
% $^1$Amsterdam University Medical Center, The Netherlands\\
% $^2$Horaizon BV, The Netherlands
% \emails
% \{j.p.belopereira, e.stroes, a.k.groen, e.levin\}@amsterdamumc.nl
% }
%
\maketitle              % typeset the header of the contribution
\begin{abstract}
One of the most common pitfalls often found in high dimensional biological data sets are correlations between the features. This may lead to statistical and machine learning methodologies overvaluing or undervaluing these correlated predictors, while the truly relevant ones are ignored. In this paper, we will define a new method called \textit{pairwise permutation algorithm} (PPA) with the aim of mitigating the correlation bias in feature importance values. Firstly, we provide a theoretical foundation, which builds upon previous work on permutation importance. PPA is then applied to a toy data set, where we demonstrate its ability to correct the correlation effect. We further test PPA on a microbiome shotgun dataset, to show that the PPA is already able to obtain biological relevant biomarkers.

\end{abstract}
%
%
% ADD REFERENCES HERE
\section{Introduction}
Measuring feature importance has often been plagued by high feature correlations. One important drawback is the lack of a theoretical definition for variable importance, in case variables are correlated \cite{gromping} \cite{gregorutti}, even in linear models \cite{gromping}. From a clinical perspective, correlated biomarkers are of high interest because they both may play a role in a shared biological pathway identified by the model and yet exhibit different behaviour in other circumstances.
%To circumvent this definition problem, feature selection techniques or dimensionality reduction methods are some of the most commonly used \cite{stats}. 
The method proposed in this paper, which will be referred to as \textit{pairwise permutation algorithm} (PPA), allows us to calculate the importance of features without having to rely on the previously mentioned selection approaches. Highly correlated features, which have a similar relation with the output value, should have close importance ranks since they explain the same variability in the data. The \textit{pairwise permutation algorithm} aims to provide feature importance values while avoiding the use of aggressive pre-selection techniques, since these techniques might remove relevant information from the data. It also manages to retain model interpretability by generating an importance value per feature, even when applied to black box models. Moreover, when working with highly dimensional biological data sets, it is simply not feasible to try and address each of the correlations in the data individually. 

\section{Related Work}

In this work, we focus on model-agnostic procedures which can be divided into local and global methods.
Local-based methods such as LIME 
%(Local Interpretable Model agnostic Explanations)
 and its variants \cite{lime,gse} attempt to explain predictions on single data points by perturbing it and building a simple, yet interpretable model on the perturbed predictions. Similarly, SHAP 
 %(SHapley Additive exPlanations)
 \cite{shap}, offers a local explanation based on the additional prediction value each feature has when adding it to all possible feature subsets. 
Unlike local-based methods, global methods are concerned with determining the overall model behaviour and what features it values for its prediction. For example, in clinical research, the goal is to determine biomarkers that can identify a condition in the general population, or potential targets for novel drug development. Therefore, in this setting, we are mainly concerned with a more holistic view of feature importance i.e. global. A notable example is that of permutation importance which was first introduced by Breiman \cite{rf} in random forests as a way to  understand the interaction of variables that is providing the predictive accuracy. Suppose that for a certain feature $i$ in data-set $\bold{X}$, we randomly permute the instances' values, and denote the resultant data-set by $\bold{X}_{i}^{\pi}$. Permutation importance is defined as the difference in the expected model loss on the original dataset and the original one:

\begin{equation}\label{fisher}
    PI_{\{i\}}(f) := E\left[f(\bold{X}_{i}^{\pi})\right] - E\left[f(\bold{X})\right]
\end{equation}

For random forests, there is already available work that analyzes the behaviour of this permutation importance, including the cases when high correlations are present. Gregorutti et al \cite{gregorutti} provided a theoretical description of the effect of correlations on the permutation importance, a phenomenon already observed by Toloși and Lengauer \cite{gregorutti} \cite{tolosi}. Furthermore, a feature selection procedure was introduced, which was more efficient in selecting important, highly correlated variables\cite{gregorutti}.
Strobl et al  showed that the larger feature importance values for correlated predictors in random forests were due to the preference for such predictors in the early splits of the trees. A new conditional permutation-based feature importance calculation was suggested, in order to circumvent this inflation, as well as the depreciation for its correlated predictor \cite{cond_rf}. Furthermore, Hooker and Mentch proposed the 'permute and relearn' approach \cite{stop_permuting}. Based on this approach we define the relearned permutation importance as

\begin{equation}\label{relearn}
    PI_{j}^{\pi L} = E\left[f^{\pi j}(\bold{X_{t}})\right] - E\left[f(\bold{X_{t}})\right]
\end{equation}

In which $f^{\pi j}$ is the model trained on the train dataset $\bold{X^{\pi j}}$, in which feature j is permuted, $f$ the model trained on the original train dataset $\bold{X}$ and $\bold{X_{t}}$ the test dataset. One drawback of this approach was also mentioned in the context of correlated features, as this resulted in the compensation effect, in which the importance of the correlated features was reduced \cite{stop_permuting}.
Local based methods, such as the ones introduced earlier, are focused on the contribution of each feature towards individual predictions, whereas permutation importance gives us a more broad estimation, since it is based on the overall accuracy of the model. While the former approach provides a higher degree of interpretability, the latter is usually more appropriate in a research environment, in which the aim would be to discover new leads which could help researchers to investigate the underlying biological mechanisms.

\section{Pairwise Permutations Algorithm}

\subsection{Notation}

We will refer to a single instance of the data-set as instance or point interchangeably throughout the paper. We denote matrices, 1-dimensional arrays and scalars with capital bold and regular text, respectively (e.g. $\mathbf{X},\; \mathbf{x},\; \alpha$). Matrices' columns and rows will be denoted by $\mathbf{X}[:, i]$ and $\mathbf{X}[j, :]$, respectively. The expected loss of a function given by: $\frac{1}{N}\sum_{i=1}^{N}l\left[y,f(\bold{x}_i)\right]$ will be denoted by $E\left[f(\bold{X})\right]$.

\subsection{Intuition}
Features that are equally important for the output value should have similar feature importance ranks, and these should not be affected by feature correlation. In an attempt to prevent the compensation effect for correlated features mentioned by Hooker and Mentch, we have chosen to permute all the feature pairs and calculate the corresponding permutation importance of the pair. 
A key assumption in our method is that the higher the correlations, the larger should be the correction to that feature individual importance.

%\begin{figure}
%    \centering
%    \includegraphics[width=.6\textwidth]{diagram_updated.png}
%    \caption{How the PPA calculates the loss for each feature pair.}
%    \label{fig:feature_importance_microbes}
%\end{figure}

\subsection{Definition}

In this section, we define the pairwise permutation importance (PPI) as the weighted average of the permutation importance values, computed using the 'permute and relearn' approach defined in Equation \ref{relearn}. The correlations between the feature pairs will act as the weights. Let $\bold{R}$ be the correlations matrix between all the features and $\bold{R}_{i,j}$ the correlation value between features i and j. Let $PI_{i,j}$ define the relearn permutation importance (see Equation \ref{relearn}) when both the feature $i$ and $j$ have been permuted together, and $PI_{i,i}$ the relearn permutation importance, when only feature $i$ is permuted.

\begin{equation}\label{pairwise_perm_def}
PPI_i =\frac{1}{\underbrace{\sum\limits_{j=1}^M |\bold{R}_{i,j}|}_{q}} \underbrace{\left(PI_{i,i} + \sum\limits_{\substack{j=1\\j \neq i}}^M |\bold{R}_{i,j}|\cdot PI_{i,j}\right)}_{p}
%=\frac{PI_\mathbf{A}+\rho(A,B).PI_\mathbf{A,B}+\rho(A,C).PI_\mathbf{A,C}+...}{\%sum\limits_{i=1}^m \rho(A,Feat_i)
\end{equation}

Note that when a feature has no correlations in the data, according to the previous equation, the PPI will actually follow the relearn permutation importance, which in our terminology is the single permutation importance (SPI). Since for complex data sets with thousands of features the computational time can become infeasible ($\mathcal{O}(N^2)$), one possible simplification is to set a threshold and consider only the permutation pairs with a correlation above it. We define this procedure in algorithm \ref{alg1}.

\subsection{Expected Difference}\label{exp_diff}

It might be tempting to compute the expected loss of the model, perform the permutation analysis and then compute the difference of the expected losses. This is actually how Fisher et al. \cite{fisher} defined the permutation importance. However, we note that this procedure is sub-par as we show in the following theorem:
\begin{theorem}
	For a given function $f:\mathbb{R}^M\rightarrow\mathbb{R}$, let $\bold{X}$ and $\bold{x}$ be a sample and an instance from the domain of $f$, respectively,  $\bold{X}^{\epsilon}_i$ be $\bold{X}$ with permuted values for the r.v. $X_i$ and $\tilde{\bold{x}}$ an instance from $\bold{X}^{\epsilon}_i$. Then, for any loss function $l\left[y, f(\bold{x})\right]$ and norm function $||\cdot||:\mathbb{R}^M\rightarrow\mathbb{R}$ it holds that:\\ $\bold{E}\left[\lVert l\left[y, f(\bold{x})\right]-l\left[y, f(\tilde{\bold{x}})\right] \rVert\right]
	\geq 
	\lVert \bold{E}\left[l\left[y, f(\bold{x})\right]\right]-\bold{E}\left[l\left[y, f(\tilde{\bold{x}})\right]\right] \rVert $
\end{theorem} 
\begin{proof}
	Consider the following convex function $\varphi(x)=\lVert x\rVert$ for $x=l\left[y, f(\bold{x})\right]-l\left[y, f(\tilde{\bold{x}})\right]$. Then, by Jensen's inequality: 
	\begin{equation*}
	\begin{aligned}
	\bold{E}\left[\varphi(x)\right]\geq\varphi\left(\bold{E}\left[x\right]\right)\Leftrightarrow & \bold{E}\left[\lVert l\left[y, f(\bold{x})\right]-l\left[y, f(\tilde{\bold{x}})\right] \rVert\right] \geq \lVert\bold{E}\left[ l\left[y, f(\bold{x})\right]-l\left[y, f(\tilde{\bold{x}})\right]\right] \rVert\\
	&=
	\lVert\bold{E}\left[ l\left[y, f(\bold{x})\right]\right]-\bold{E}\left[l\left[y, f(\tilde{\bold{x}})\right]\right] \rVert 
	\end{aligned}
	\end{equation*} \qed
\end{proof}
This means that computing the expected value of the normed difference of individual loss values is more robust to non-linear relationships between the input variables then computing the difference of the normed expected loss values.

\begin{algorithm}[h!]
  \SetAlgoLined
  
  \caption{Pairwise permutations algorithm}
  \BlankLine
  \KwIn{$\bold{X}$, $\bold{X_{t}}$, $\bold{y_{test}}$, $E\left[f(\bold{X_{t}})\right]$, $\bold{R}$, $\alpha$}
  \KwOut{$\bold{v}$}
  \BlankLine
  \For{feature $i$ in $\bold{X}$}{
    $p\gets0$, $q\gets0$ (equation \ref{pairwise_perm_def})\;
    
    \For{feature $j$ in $\bold{X}$}{
        \If{$|\bold{R}_{i,j}|>\alpha$}{

        Permute the feature pair ($i$, $j$) together in $\bold{X}$\;
        Retrain the model with the permuted input data $\bold{X_{i,j}^{\pi}}$\;
        Calculate the model's error, $E\left[f^{\pi, i, j}(\bold{X_{t}})\right]$, on the test data\;
        Calculate $PI_{i,j}$ through the relearn formula $ E\left[f^{\pi, i, j}(\bold{X_{t}})\right] - E\left[f(\bold{X_{t}})\right]$\;
        \eIf{i=j}{
        $p \gets p + PI_{i,i}$\;
        }{
        $p \gets p + |\bold{R}_{i,j}| \cdot  PI_{i,j}$\;
        }
        $q\gets q + |\bold{R}_{i,j}|$\;}
    }
    $PPI_{i}\gets p / q$\;
    $\bold{v}$.append($PPI_{i}$)\;
  }
\end{algorithm}\label{alg1}

\section{Simulations with toy dataset}
To see how our new PPA would behave for correlated features, we generated a toy dataset, based on the one used by Hooker and Mentch \cite{stop_permuting}. The data was created by assuming a linear regression model:
\begin{equation}\label{linear_formula}
y_i = x_{i1} + x_{i2} + x_{i3}+ x_{i4} + x_{i5} + 0x_{i6} + 0.5x_{i7} + 0.8x_{i8} + 1.2x_{i9} + 1.5x_{i10}.
\end{equation}
This was then turned into a classification model, by generating the binary outcome y with the classification rule in Equation \ref{eq:classifrule} with with $\varepsilon \sim N(0, 0.1)$.
\begin{equation}\label{eq:classifrule}
y_i = 
\begin{cases}
1, & \text{for } y_i + \varepsilon_i \geq \overline{y}\\
0, & \text{otherwise}\end{cases},
\end{equation}
All features were generated from a multivariate normal distribution $N(0, \Sigma)$ with $\Sigma$ equal to the identity matrix, except that $\Sigma_{12} = \Sigma_{21} = \rho = 0.9$. All features were then transformed into a uniform distribution, mimicking how Hooker and Mentch \cite{stop_permuting} generated their data. In total, 1000 samples were generated.\\
In case the features are in the same scale, the coefficients in the linear model can be seen as the conditional importance of the feature on all other variables \cite{cond_rf} \cite{stop_permuting}. Therefore, based on the magnitude of the coefficients, we can rank the features on their importances, where features with the same coefficients should be equally important, while a feature with a higher coefficient should get higher importance than a feature with a lower coefficient. The order of the features should not be affected by any correlations between the features.\\
Using XGBoost with the logistic loss function as the classification algorithm \cite{xgboost1,xgboost2}, we performed $50$ stratified shuffle splits ($70\%$train/$30\%$test) and measured the ROC AUC after adding a noise feature to the dataset and standard scaling it. We found the optimal hyperparameters using a 5-fold cross validation grid search. 
To compute the PPIs, a correlation threshold of 0.3 was used. Also, the SPIs were obtained.

\subsection{Results}
\begin{figure}[b!]
\centering
\begin{subfigure}[b]{0.3\linewidth} % Simulation 7
\includegraphics[width=\linewidth]{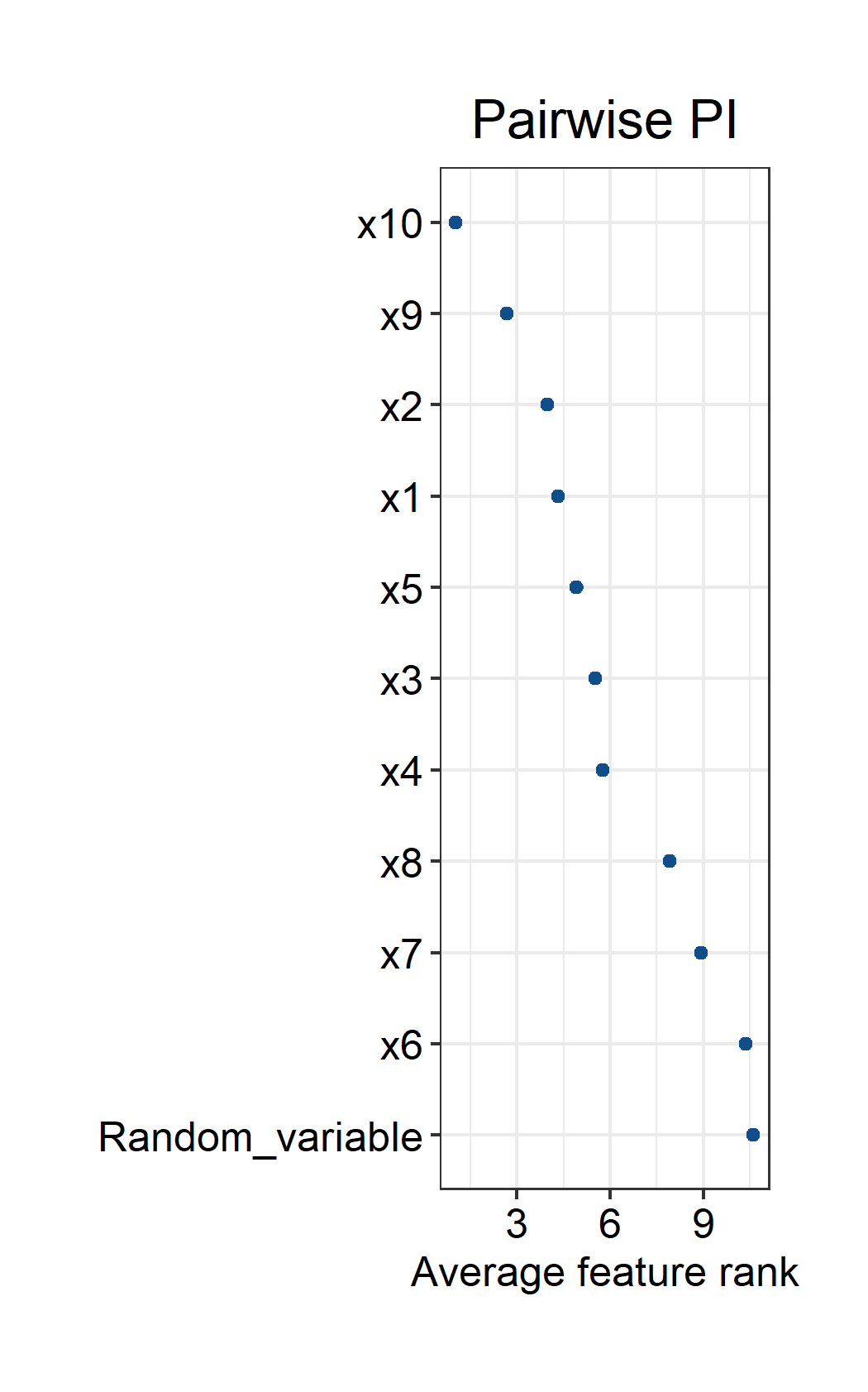}
\caption{$x_1$ and $x_2$ with $\rho = 0.9$}
\label{fig:lin_features_a}
\end{subfigure}
\begin{subfigure}[b]{0.3\linewidth} % Simulation 7
\includegraphics[width=\linewidth]{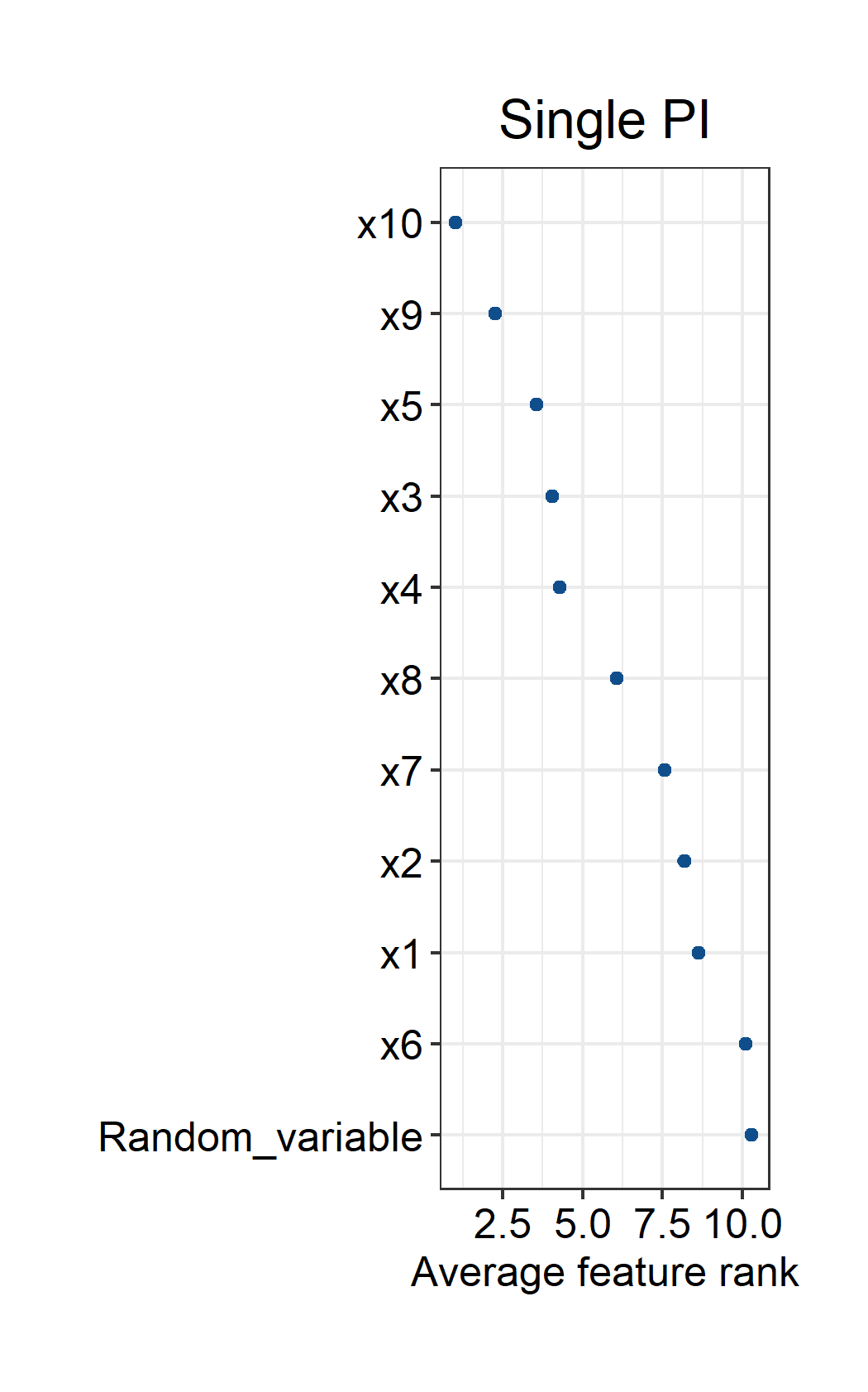}
\caption{$x_1$ and $x_2$ with $\rho = 0.9$}
\label{fig:lin_features_b}
\end{subfigure}
\begin{subfigure}[b]{0.3\linewidth}
\includegraphics[width=\linewidth]{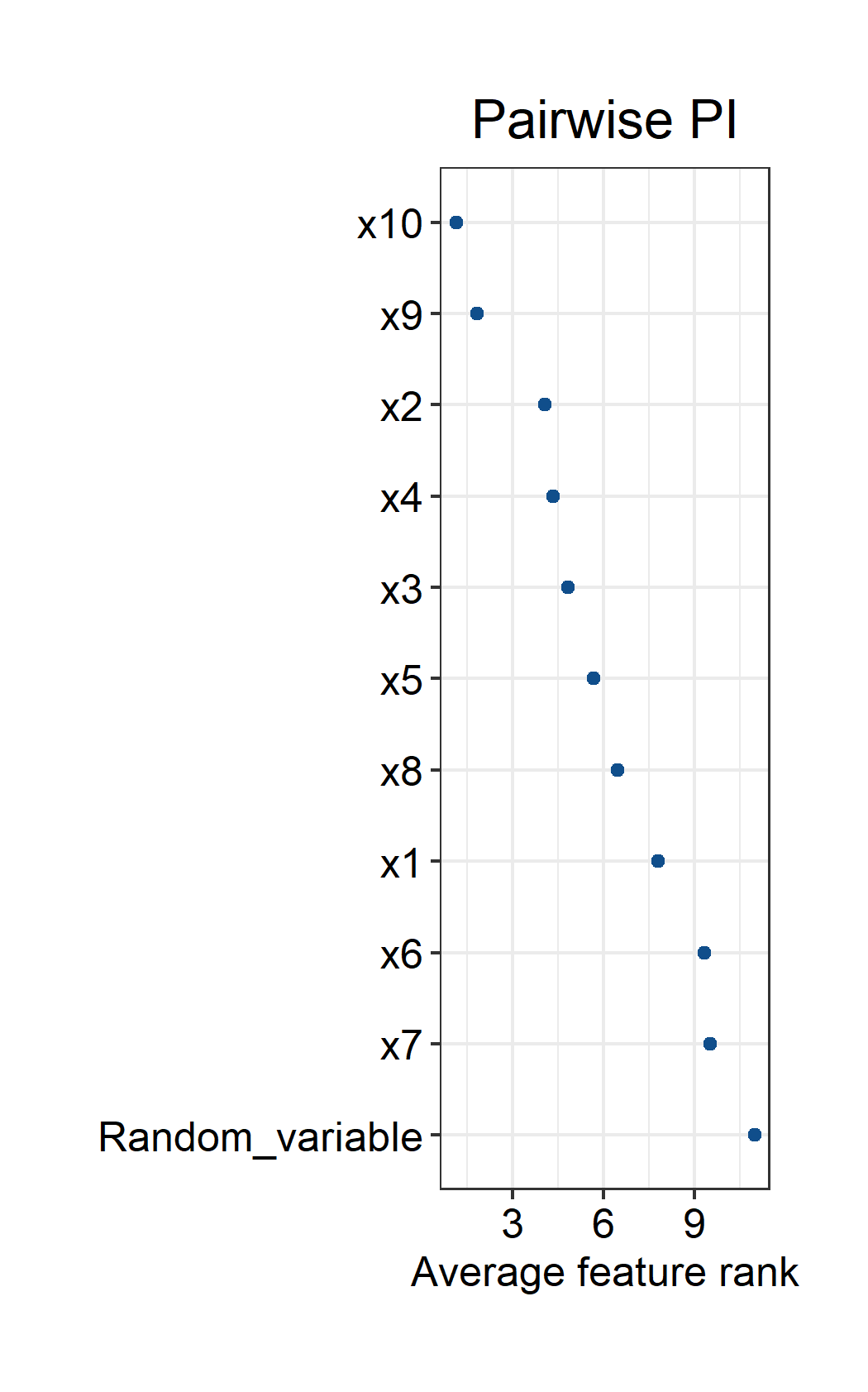}
\caption{$x_1$ and $x_6$ with $\rho = 0.9$}
\label{fig:lin_features_c}
\end{subfigure}
\caption{Average rank $\pm$ standard error for each feature based on the pairwise permutation importance algorithm for (a) and (c) and the single permutation importance algorithm for (b). } 
\label{fig:lin_features}
\end{figure}

The classification model obtained an average AUC of $0.97 \pm 0.01$. As shown by the average feature ranks in Figure \ref{fig:lin_features_a}, our new PPA is able to retrieve the right order of feature importances, in which $x_{10}$ is clearly the most important one, followed by $x_9$. As expected, $x_6$ and the random variable are identified as the least important features.
The results for the SPI are shown in Figure \ref{fig:lin_features_b}. It is clearly shown that the PPA outperforms this approach, as the SPI decreased the importance of the correlated features $x_1$ and $x_2$ and was not able to retrieve the right order of feature importances. This was also observed for the random forest algorithm by Hooker and Mentch \cite{stop_permuting}.
The toy dataset showed that in case two features have the same coefficient in the linear model and are correlated, the PPA is able to retrieve the right order for the feature importances. We also analysed the effect of a correlation of $\rho = 0.9$ between $x_1$ and $x_6$, by changing the covariance matrix $\Sigma$ to $\Sigma_{16} = \Sigma_{61} = 0.9$ and setting the correlation between $x_1$ and $x_2$ to 0. This represents a case in which an important feature is correlated to an irrelevant feature. However, we saw in this case that the importance of $x_1$ was decreased by $x_6$, while the importance of $x_6$ was increased by $x_1$, as shown in Figure \ref{fig:lin_features_c}. This could be expected as the grouped importance is shared equally between both features, while in the case of features with different importances, this might not be the right assumption. In this case, the PPA may not be the appropriate choice 

\section{Microbial biomarkers for Type 2 Diabetes Mellitus}
In this section we test the PPA on a real-world dataset, specifically microbiome data. The goal is to obtain biologically relevant markers. 
Therefore, we downloaded the Qin 2012 microbiome dataset from MLRepo \cite{mlrepo}, \cite{T2DQin}. This curated classification dataset contained shotgun data for 124 samples, representing Chinese healthy controls (n = 59) and Type 2 Diabetes Mellitus (T2D) patients (n = 65). For full details of the preprocessing of the raw sequence reads for datasets in MLRepo, see \cite{mlrepo}.
We used the same procedure as in the previous section with some additional preprocessing. First, the read counts were rarefied to 28 358 reads per sample, which was the lowest observed number of reads in a sample. After that, features with less than 6 reads per sample on average, representing a relative abundance of 0.02\%, were removed. The final dataset consisted then of 124 samples with 377 microbial OTUs.  
\subsection{Results}
The classification model was able to achieve an average ROC AUC score of 0.92 $\pm$ 0.05, as depicted in figure \ref{fig:ROC_AUC}. Figure \ref{fig:feature_importance_microbes} represents the top 15 most predictive OTUs in the classification model. 
Analyzing these OTUs (and several more beyond the top 15) primarily highlights 2 main patterns. The strongest pattern observed in the data, most likely represents an effect that T2D has on the dietary behavior of these Chinese T2D patients. \textit{Lactobacillus acidophilus, Acidaminococcus intestini} and \textit{Anaerostipes caccae} are strongly associated with T2D and with each other in this dataset. A regular dose of \textit{L. acidophilus} is commonly recommended in Chinese Medicine \cite{ChineseHand}. Fermented soybean products are popular in China (i.a.) and various of these products commonly contain \textit{L. acidophilus} \cite{Marcus_1}, \cite{Marcus_2}, \cite{Marcus_3}. Indeed, there is evidence that supports the beneficial claims regarding these fermented products and T2D \cite{Marcus_4}, \cite{Marcus_5}. Trans-aconitic acid in the urine is a biomarker for the consumption of soy products \cite{Marcus_6} and \textit{A. intestini} is known to be able to oxidise trans-aconitate \cite{Marcus_7} converting it to acetate. \textit{A. caccae}, is an acetate and lactate consuming butyrate producer. Cross-feeding interactions between \textit{L. acidophilus} and \textit{A. caccae} have been analyzed in detail in vitro \cite{Marcus_8}. Other butyrate producing species, like \textit{Roseburia intestinalis}, can have similar cross-feeding interactions \cite{Marcus_9} but were not part of this specific pattern, but with the 2nd main pattern (see below), suggesting that \textit{A. caccae} was part of same fermented soybean product popular with, or given to, these Chinese T2D patients that likely also contained \textit{L. acidophilus} and \textit{A. intestini}.

\begin{figure}[tb]
\centering
  \begin{subfigure}[b]{0.43\textwidth}
    \includegraphics[width=0.95\textwidth]{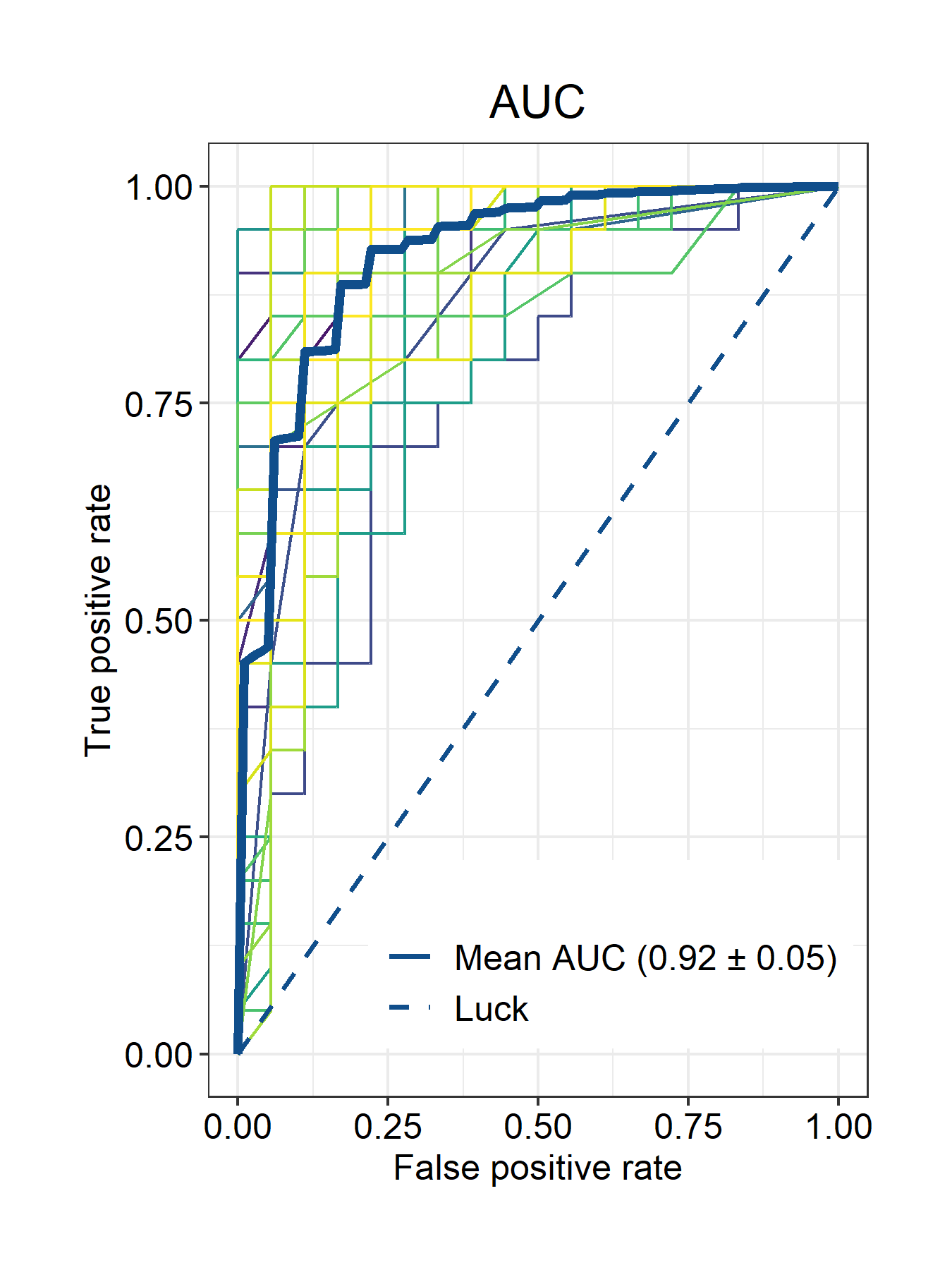}
    \caption{The ROC-AUC plot.}
    \label{fig:ROC_AUC}
  \end{subfigure}
  \begin{subfigure}[b]{0.465\textwidth}
    \includegraphics[width=0.95\textwidth]{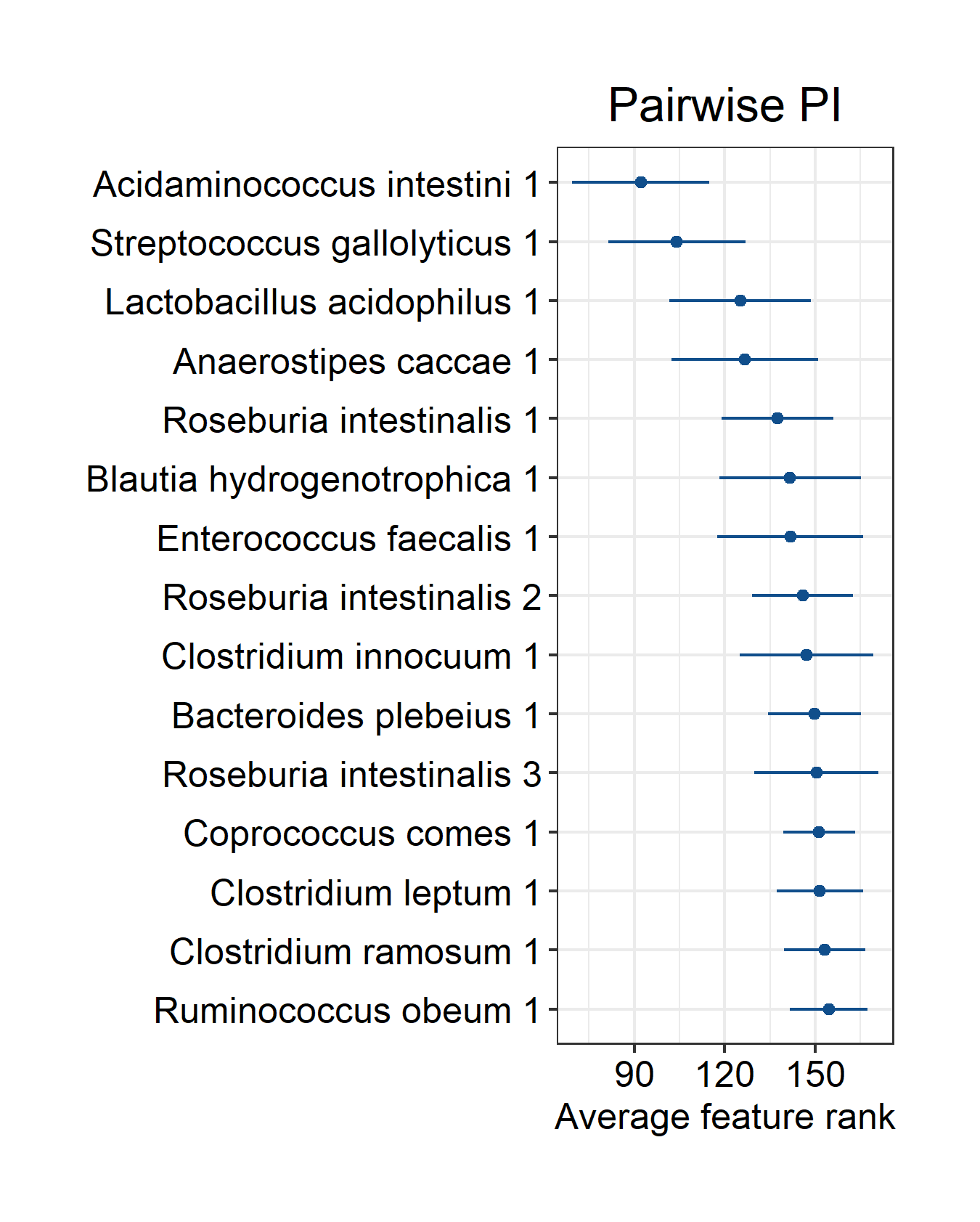}
    \caption{Top 15 feature importance ranks.}
    \label{fig:feature_importance_microbes}
  \end{subfigure}
\caption{(a) Individual ROC AUC curves for each shuffle and average ROC AUC plot for all shuffles. (b) Average rank $\pm$ standard error for the top 15 ranked features based on the pairwise permutation importance algorithm. }
\end{figure}

The second pattern involves several butyrate producers (the \textit{Roseburia, Faecalibacterium, Coprococcus} genera, several \textit{Eubacterium} species and \textit{Anaerostipes hadrus}) in a cross-feeding relationship with various acetate producing dietary fibre degrading species (\textit{Blautia} and \textit{Ruminococcus} representatives). This cluster of species is generally found to be negatively associated with T2D, not just in this study throughout the diabetes microbiome field \cite{T2DQin,Marcus_10,Marcus_11,Marcus_12,Marcus_13}. Insufficient butyrate production has been associated with both T1D and T2D development both in rats, mice and in humans \cite{Marcus_14}, \cite{Marcus_15}, \cite{Marcus_16}, \cite{Marcus_17}. Besides being used by colonocytes as a primary energy source \cite{Marcus_18} butyrate is a powerful inhibitor of histone deacetylase, which has emerged as a target in the control of insulin resistance \cite{Marcus_19}, \cite{Marcus_20}, \cite{Marcus_21}. Animal and in vitro studies have generally found a beneficial effect of butyrate and acetate on glucose homeostasis and insulin sensitivity \cite{Marcus_22}. 
%
%In conclusion, these results showcase already the potential of PPA for finding relevant biological signals in complex datasets, as the PPA was able to select biomarkers with known relations to T2D.
% 
\section{Conclusions}
In this paper, we have set a first step in correcting the compensation effect, observed for 'permute and relearn' permutation importances in case correlated features are present. Our new PPA is able to obtain the right ranking of features, when two features are highly correlated and have the same importance, stated by the magnitude of their coefficient, in linear models. Furthermore, while not yet optimal for correlations between more than 2 features or correlated features with unequal importance related to the output variable, our PPA is already able to obtain relevant biological insights in a Chinese Type 2 Diabetes microbiome dataset.

\bibliographystyle{splncs04}

\section{Acknowledgments}
We would like to thank Manon Balvers for helping with computational experiments.

\end{document}